\documentclass{article}
\usepackage{hyperref}
\usepackage{amsmath}
\usepackage{amsfonts}
\usepackage{amssymb}
\usepackage{pdfsync}
\usepackage[round]{natbib}
\usepackage{fullpage}
\newif\ifdraft \drafttrue
\newif\iffull \fulltrue

\usepackage{algorithmic}
\usepackage[ruled]{algorithm2e}
\usepackage{amsthm}
\newtheorem{theorem}{Theorem}
\newtheorem{lemma}[theorem]{Lemma}
\newtheorem{corollary}{Corollary}[theorem]
\usepackage{bbm}
\newtheorem*{remark}{Remark}

\theoremstyle{definition}
\newtheorem{definition}{Definition}[section]
\usepackage{graphicx}
\usepackage{xcolor}
\usepackage[english]{babel}

\usepackage{xspace}
\usepackage[utf8]{inputenc}
\usepackage[T1]{fontenc}

\usepackage{color-edits}

\usepackage{caption}
\usepackage{subcaption}

\usepackage{microtype}
\usepackage{graphicx}
\usepackage{booktabs} 
\usepackage{hyperref}
\usepackage{cleveref}
\addauthor{sw}{blue}

\DeclareMathOperator*{\argmin}{arg\,min}
\DeclareMathOperator*{\argmax}{arg\,max}
\DeclareMathOperator{\Lap}{Lap}
\DeclareMathOperator{\separator}{sep}

\newcommand{\sepQ}{\separator(\cQ)}
\newcommand{\bbE}{\mathbb{E}}
\newcommand{\cA}{\mathcal{A}}

\newcommand{\cQ}{\mathcal{Q}}
\newcommand{\cM}{\mathcal{M}}
\newcommand{\cX}{\mathcal{X}}
\newcommand{\cY}{\mathcal{Y}}
\newcommand{\cZ}{\mathcal{Z}}
\newcommand{\cL}{\mathcal{L}}
\newcommand{\cR}{\mathcal{R}}
\newcommand{\bD}{\mathbbm{D}}

\newcommand{\Ex}[1]{\mathbb{E}\left[ #1 \right]}
\newcommand{\E}[2]{\mathbb{E}_{#1}\left[ #2 \right]}
\newcommand{\inner}[2]{\left\langle #1, #2 \right\rangle}

\newcommand{\pr}[1]{\text{Pr}\left[ #1 \right]}
\newcommand{\algoA}{\text{Non-Convex-FTPL }}
\newcommand{\algoB}{\text{Separator-FTPL }}
\newcommand{\regretA}{R_{\text{NC}}(T)}
\newcommand{\regretB}{R_{\text{Sep}}(T)}
\newcommand{\DQ}{\mbox{{\sf DualQuery}}\xspace}

\newcommand{\eps}{\varepsilon}
\newcommand{\dqrs}{\mbox{{\sf DQRS}}\xspace}
\newcommand{\dq}{\mbox{{\sf DualQuery}}\xspace}
\newcommand{\mwem}{\mbox{{\sf MWEM}}\xspace}
\newcommand{\fem}{\mbox{\sf FEM}\xspace}
\newcommand{\sfem}{\mbox{{\sf sepFEM}}\xspace}
\newcommand{\hdmm}{\mbox{{\sf HDMM}}\xspace}
\newcommand{\actions}{\cA}

\newcommand{\data}{\mathrm{data}}
\newcommand{\query}{\mathrm{qry}}

\usepackage{tabularx}

\allowdisplaybreaks
\begin{document}



\title{New Oracle-Efficient Algorithms for Private Synthetic Data Release}

\author{Giuseppe Vietri\thanks{Department of Computer Science and Engineering, University of Minnesota. Supported by the GAANN fellowship from the U.S. Department of Education.} \and
Grace Tian \thanks{Harvard University.}
\and Mark Bun\thanks{Boston University. Supported by NSF grant CCF-1947889. Part of this work was done at the Simons Institute for the Theory of Computing, supported by a Google Research Fellowship.}\and
Thomas Steinke\thanks{IBM Research.}\and
Zhiwei Steven Wu\thanks{Department of Computer Science and Engineering, University of Minnesota. Supported in part by a Google Faculty Research Award, a J.P. Morgan Faculty Award, a Mozilla research grant, and a Facebook Research Award.}}
\date{}
\maketitle


\begin{abstract}
We present three new algorithms for constructing differentially private synthetic data---a sanitized version of a sensitive dataset that approximately preserves the answers to a large collection of statistical queries. All three algorithms are \emph{oracle-efficient} in the sense that they are computationally efficient when given access to an optimization oracle. Such an oracle can be implemented using many existing (non-private) optimization tools such as sophisticated integer program solvers. While the accuracy of the synthetic data is contingent on the oracle's optimization performance, the algorithms satisfy differential privacy even in the worst case. For all three algorithms, we provide theoretical guarantees for both accuracy and privacy. Through empirical evaluation, we demonstrate that our methods scale well with both the dimensionality of the data and the number of queries. Compared to the state-of-the-art method High-Dimensional Matrix Mechanism \cite{McKennaMHM18}, our algorithms provide better accuracy in the large workload and high privacy regime (corresponding to low privacy loss $\eps$).
\end{abstract}


\section{Introduction}

The wide range of personal data collected from individuals has
facilitated many studies and data analyses that inform decisions
related to science, commerce, and government policy. Since many of
these rich datasets also contain highly sensitive personal
information, there is a tension between releasing useful information
about the population and compromising the privacy of individuals. In
this work, we consider the problem of answering a large collection of
statistical (or linear) queries subject to the constraint of
differential privacy. Formally, we consider a data domain
$\cX=\{0,1\}^d$ of dimension $d$ and a dataset $D \in \cX^n$
consisting of the data of $n$ individuals. Our goal is to
approximately answer a large class of statistical queries $\cQ$ about
$D$. A statistical query is defined by a predicate
$\phi \colon \cX \rightarrow [0,1]$, and the query
$q_\phi:\cX^n\rightarrow [0,1]$ is given by
$q_\phi(D) = \frac{1}{n}\sum_{i =1}^n \phi(D_i)$ and an approximate answer
$a\in [0,1]$ must satisfy $|a-q_\phi(D)| \le \alpha$ for some accuracy
parameter $\alpha>0$. To preserve privacy we work under the constraint
of differential privacy \citep{DworkMNS06}.  Privately answering
statistical queries is at the heart of the 2020 US Census release
\citep{Abowd18} and provides the basis for a wide range of private data
analysis tasks. For example, many machine learning algorithms can be
simulated using statistical queries~\citep{Kearns98}.

An especially compelling way to perform private query release is to
release \emph{private synthetic data} -- a sanitized version of the
dataset that approximates all of the queries in the class
$\cQ$. Notable examples of private synthetic data algorithms are the
SmallDB algorithm \citep{BlumLR08} and the private multiplicative
weights (PMW) mechanism~\citep{HR10} (and its more practical variant
the multiplicative weights exponential
mechanism \mwem~\citep{HardtLM12}), which can answer exponentially
many queries and achieves nearly
optimal sample complexity~\citep{BunUV18}. Unfortunately, both
algorithms involve maintaining a probability distribution over the
data domain $\cX=\{0,1\}^d$, and hence suffer exponential (in $d$) running time. Moreover, under standard cryptographic assumptions,
this running time is necessary in the worst
case~\citep{Ullman16,UllmanV11}. However, there is hope that these worst-case intractability results do not apply to real-world datasets.

To build more efficient solutions for constructing private synthetic
data, we consider \emph{oracle efficient} algorithms that rely on a
black-box optimization subroutine. The optimization problem is NP-hard
in the worst case. However, we invoke practical optimization
heuristics for this subroutine (namely integer program solvers such as
CPLEX and Gurobi). These heuristics work well on many real-world
instances.  Thus the algorithms we present are more practical than the
worst-case hardness would suggest is possible.  While the efficiency
and accuracy of our algorithms are contingent on the solver's
performance, differential privacy is guaranteed even if the solver
runs forever or fails to optimize correctly.

\paragraph{Overview of our results.} To describe our algorithms, we
will first revisit a formulation of the query release problem as a
zero-sum game between a data player who maintains a distribution
$\widehat D$ over $\cX$ and a a query player who selects queries from
$\cQ$~\citep{HsuRU13,gaboardi2014dual}. Intuitively, the data player aims
to approximate the private dataset $D$ with $\widehat D$, while the query
player tries to identify a query which distinguishes between $D$ and
$\widehat D$. Prior work~\cite{HsuRU13,gaboardi2014dual} showed that any
(approximate) equilibrium for this game gives rise to an accurate
synthetic dataset. To study the private equilibrium computation within
this game, we consider a \emph{primal} framework and a \emph{dual}
framework that enable us to unify and improve on existing algorithms.

In the {primal framework}, we perform the equilibrium computation via
the following \emph{no-regret dynamics}: over rounds, the data player
updates its distribution $\widehat D$ using a no-regret online learning algorithm,
while the query player plays an approximate best response. The
algorithm \mwem in prior work falls under the primal framework with
the data player running the multiplicative weights (MW) method as the
no-regret algorithm, and the query player privately responding using the
exponential mechanism \citep{mcsherry2007mechanism}. However, since the MW method maintains an entire
distribution over the  domain $\cX$, \mwem runs in exponential
time even in the best case. To overcome this intractability, we
propose two new algorithms \fem and \sfem that follow the same
no-regret dynamics, but importantly replace the MW method with two
variants of the follow-the-perturbed-leader (FTPL) algorithm
\citep{KalaiV05}---\algoA \citep{online2019suggala} and \algoB
\citep{syrgkanis2016efficient}---both of which solve a perturbed
optimization problem instead of maintaining an exponential-sized
distribution. \fem achieves an error rate of
$$\alpha=\widetilde O\left(d^{3/4} \log^{1/2}|\cQ| / n^{1/2}\right),$$ and \sfem achieves a
slightly better rate of
$$\alpha=\widetilde O\left( d^{5/8} \log^{1/2}|\cQ| / n^{1/2}\right),$$ although
the latter requires the query class $\cQ$ to have a structure called a
small separator set. In contrast, \mwem attains the error rate
$\alpha=\widetilde O\left(d^{1/4}
  \log^{1/2}|\cQ|/n^{1/2}\right)$. Although the accuracy analysis
requires repeated sampling from the FTPL distribution (and thus
repeatedly solving perturbed integer programs), our experiments show
that the algorithms remain accurate even with a much lower number of
samples, which allows much more practical running time.

We then consider the \emph{dual} formulation and improve upon the
existing algorithm \dq~\citep{gaboardi2014dual}.  Unlike \mwem, \dq
has the query player running MW over the query class $\cQ$, which is
often significantly smaller than the data domain $\cX$, and has the
data player playing best response, which can be computed non-privately
by solving an integer program. Since the query player's MW
distribution is a function of the private data, \dq privately
approximates this distribution with a collection of samples drawn from
it. Each draw from the MW distribution can be viewed as a single
instantiation of the exponential mechanism, which provides a bound on
the privacy loss. We improve \dq by leveraging the observation that
the MW distribution changes slowly between rounds in the no-regret
dynamics. Thus can reuse previously drawn queries to approximate the
current MW distribution via rejection sampling. By using this
technique, our algorithm \dqrs (\dq with rejection sampling) reduces
the number of times we draw new samples from the MW distribution and
also the privacy loss, and hence improves the privacy-utility
trade-off. We theoretically demonstrate that \dqrs improves the
accuracy guarantee of \dq. Specifically \dqrs~attains accuracy
$$\alpha = \widetilde{O}\left(\frac{\log(|\cX|/\beta) \cdot
    \log^{3}(|\cQ|)}{n^{2} }\right)^{1/5}$$ whereas \DQ~attains
accuracy
$\alpha = \widetilde{O}\left(\frac{\log(|\cX|/\beta) \cdot
    \log^{3}(|\cQ|) }{n^{2} }\right)^{1/6}$. Even though the dual
algorithms \dq~and \dqrs~have worse accuracy performance than the
primal algorithms \fem~and \sfem, the dual algorithms run
substantially faster, since they make many fewer oracle calls. Thus we
observe a tradeoff not only between privacy and utility but also with
computational resources.


In addition to our theoretical guarantees, we perform a comprehensive
experimental evaluation of our algorithms. As a benchmark, we use
the state-of-the-art High-Dimensional Matrix Mechanism
(HDMM)~\citep{McKennaMHM18}; HDMM is being deployed in practice by the US Census Bureau \citep{Kifer19}. We perform our experiments with the standard ADULT and LOANS datasets and use $k$-way conjunctions as a query workload. We compare both algorithms on different workload sizes and different privacy levels.  Our experiments show that as we increase the workload size \fem performs better compared to \hdmm. Similarly, \fem~ does better when we increase the privacy level.
These results support our theoretical analysis.
%

\begin{table}
\caption{Error bound Comparison}
\begin{center}
\begin{sc}
\begin{tabular}{lccl}
\toprule
Algorithm & $\alpha$ \\
\midrule
 \mwem &  $O\left(\frac{d^{1/4}\log^{1/2}|\cQ| \log^{1/2}(1/\delta)}{n^{1/2}\varepsilon^{1/2}}\right)$  \\
 \DQ
  & $ O \left( \frac{d^{1/6} \log^{1/2}{|\cQ|} \log^{1/6}(1/\delta)}{n^{1/3}		\varepsilon^{1/3}}\right) $ \\
 \fem
  & ${O} \left( \frac{d^{3/4} \log^{1/2} |\cQ| \log^{1/2}(1/\delta)} {n^{1/2} \varepsilon^{1/2}} \right)$\\
\sfem
  & $O \left( \frac{d^{5/8} \log^{1/2} |\cQ| \log^{1/2}(1/\delta)} {n^{1/2} \varepsilon^{1/2}} \right)$ \\
 \dqrs
 &  ${O} \left(\frac{d^{1/5}\log^{3/5}|\cQ|\log^{1/5}(1/\delta)}{n^{2/5}\eps^{2/5}} \right)$ \\
  \bottomrule
\end{tabular}
\end{sc}
\end{center}
Parameters: $(\eps,\delta)$-differential privacy, $n$ data points of dimension $d$, query class $\cQ$, accuracy $\alpha$.
\end{table}

\subsection{Additional related work}

Aside from the aforementioned \dq\ algorithm~\citep{gaboardi2014dual},
several works on differentially private query release and synthetic
data generation are described in, or can be placed in, the framework
of oracle-efficient algorithms. One example is the Projection
Mechanism~\citep{NikolovTZ13} and extensions
thereof~\citep{Nikolov15,DworkNT15,BlasiokBNS19} in which each
projection step can be approximately implemented via a non-private
optimization subroutine. This line of work focuses on the average
error over the queries, rather than the maximum error as we do.

The notion of oracle-efficiency for differential privacy was
formalized in a recent work of~\citet{NeelRW19} who introduced
techniques for oracle-efficient private synthetic data generation even
for exponentially large classes of queries. A more recent work by \citet{NRVW} provides oracle-efficient methods for privately solving certain classes of non-convex optimization problems. In both \citet{NeelRW19} and \citet{NRVW}, the privacy guarantees of their algorithms either rely on the exact optimality or certifiability of the oracle. All of our algorithms satisfy differential privacy even if we implement the optimization oracles with a heuristic that satisfies neither condition.

In Section~\ref{sec:experiments}, we compare the performance of our algorithms against other practical algorithms for synthetic data generation. The benchmark we use is the High-Dimensional Matrix Mechanism~\citep{McKennaMHM18} which itself builds on the Matrix Mechanism~\citep{LiMHMR15} but is more efficient and scalable. Given a workload of queries $\cQ$, this algorithm uses optimization routines (in a significantly different way than ours) to select a different set of ``strategy queries'' which can be answered with Laplace noise. Answers to the original queries in $\cQ$ can then be reconstructed by combining the noisy answers to these strategy queries.

The study of oracle-efficiency also has a rich history in machine learning and optimization outside of differential privacy~\citep{BeygelzimerDHLZ05, BalcanBBCLS08, BeygelzimerDLM16, BenTalKHM15, HazanK16}. In particular, a number of works have sought to design oracle-efficient fair algorithms~\citep{AgarwalBDLW18, AlabiIK18, KearnsNRW18}.


\begin{section}{Preliminaries}\label{sec:prelim}

\begin{definition}[Differential Privacy (DP)] A randomized algorithm $\cM: \cX^* \to \cR$ satisfies $(\eps, \delta)$-differential privacy (DP) if for all databases $x, x'$ differing in at most one entry, and every measurable subset $S \subseteq \cR$, we have
	\[\Pr[\cM(x) \in S] \le e^{\eps} \Pr[\cM(x') \in S] + \delta.\]
If $\delta = 0$, we say that $\cM$ satisifies  $\eps$-diffrential privacy.
\end{definition}

To facilitate our privacy analysis, we will rely on the privacy notion
of \emph{zero-concentrated differential privacy} (zCDP), which
provides a simpler composition theorem.

\begin{definition}[Zero Concentrated Differential Privacy(zCDP)
\cite{BunS16}]
A mechanism $\cM:\cX \rightarrow R$ is
$(\rho)$-zero-concentrated differentially private if for
all neighboring datasets $x,x'\in \cX^*$, and all
$\alpha\in(0,\infty)$ the following holds
\begin{align*}
\bD_\alpha\left(M(x)||M(x')\right) \leq \rho \alpha
\end{align*}
where $\bD_\alpha$ is the $\alpha$-R\'enyi divergence
between the distribution $M(x)$ and the distribution $M(x')$.
\end{definition}

We can relate guarantees of DP and zCDP using the following lemmas.

\begin{lemma}[DP to zCDP \cite{BunS16}]\label{lem:dptocdp}
If $\cM$ satisfies $\eps$-differential privacy,
then $\cM$ satisfies $\left( \tfrac{1}{2}\eps^2 \right)$-zCDP.
\end{lemma}
\begin{lemma}[zCDP to DP \cite{BunS16}]\label{lem:cdpTodp}
If $\cM$ provides $\rho$-zCDP, then
$\cM$ is $\left(\rho + 2 \sqrt{\rho\log(1/\delta)}, \delta \right)$-DP
for $\delta>0$.
\end{lemma}
\begin{lemma}[zCDP composition \cite{BunS16}]\label{lem:zcdpcomposition}
Let $\cM:\cX^*\rightarrow \cY$ and $M':\cX^*\rightarrow\cZ$
be randomized algorithm. Suppose that $\cM$ satisfies $\rho$-zCDP
and $\cM'$ satisfies $\rho'$-zCDP. Define $\cM'':\cX\rightarrow\cY\times\cZ$
by $\cM''(x)=\left(\cM(x), \cM'(x) \right)$. Then $\cM''$ satisfies
$(\rho+\rho')$-zCDP.
\end{lemma}

We will use the exponential mechanism as a key component in our design
of private algorithms.


\begin{definition}[Exponential Mechanism \cite{mcsherry2007mechanism}]
\label{def:em}
Given some database $x$, arbitrary range $\cR$, and score function $S:\cX^{*}\times\cR\rightarrow \cR$, the exponential mechanism $\cM_E(x, S, \cR, \rho)$ selects and outputs an element $r \in \cR$ with probability proportional to
$$ \exp{\left( \frac{\rho S(x, r)}{2 \Delta_S} \right)},$$
where $\Delta_S$ is the sensitivity of $S$, defined as
$$\Delta_S = \max_{D, D': |D\triangle D'| = 1, r \in R} |S(D, r) - S(D', r)|. $$
\end{definition}

\begin{lemma}[\cite{mcsherry2007mechanism}]\label{lem:EMprivacy}
  The exponential mechanism $\mathcal{M}_E(x, S, \cR, \rho)$ is
  $\left(\frac{\rho^2}{2}\right)$-zCDP.
\end{lemma}

\begin{theorem}[Exponential Mechanism Utility \cite{mcsherry2007mechanism}] \label{thm:exp-error}.
Fixing a database $x$, let ${OPT}_S(x)$ denote the max score of function $S$. Then, with probability $1-\beta$ the error is bounded by:
$${OPT}_S(x) - S(x, \cM_E(x, u, \mathcal{R},\rho)) \le  \frac{2 {\Delta}_S}{\rho} \left( \ln{|\mathcal{R}|/\beta}\right) $$
\end{theorem}

We are interested in privately releasing \emph{statistical linear queries}, formally defined as follows.

\begin{definition}[Statistical linear queries]
  Given as predicate a linear threshold function $\phi$, the linear
  query $q_\phi: \cX^n \rightarrow [0,1]$ is defined by
$$q_\phi(D) = \frac{\sum_{x\in D} \phi(x)}{|D|}$$
\end{definition}

The main query class we consider in our empirical evaluations is $3$-way marginals and $5$-way marginals. We give the definition here
%
\begin{definition}\label{def:marginals}
Let the data universe with $d$ categorical features be $\cX = (\cX_1\times\ldots\times\cX_d)$, where each $\cX_i$ is the discrete domain of the $i$th feature. We write $x_i\in\cX_i$ to mean the $i$th feature of record $x\in\cX$. A $3$-way marginal query is a linear query specified by 3 features $a\neq b\neq c\in [d]$, and a target $y\in(\cX_a\times\cX_b\times\cX_c)$, given by
\[
q_{\text{abc,y}}(x) =
\begin{cases}
1 & : x_a = y_1 \land  x_b = y_2 \land  x_c = y_3\\
0 &: \text{otherwise.}
\end{cases}
\]
Furthermore, its negation is given by
\[
\bar{q}_{\text{abc,y}}(x) =
\begin{cases}
0 & : x_a = y_1 \land  x_b = y_2 \land  x_c = y_3\\
1 &: \text{otherwise.}
\end{cases}
\]
Note that for each marginal $(a,b,c)$ there are $|\cX_a||\cX_b||\cX_c|$ queries.
\end{definition}

Finally, our algorithm will be using the following form of linear
optimization oracle. In our experiments, we implement this oracle via
an integer program solver.

\begin{definition}[Linear Optimization Oracle] \label{linopt}Given as
  input a set of $n$ statistical linear queries $\{q_i\}$ and a
  $d$-dimensional vector $\sigma$, a linear optimization oracle
  outputs
$$\hat{x} \in \argmin_{x\in\{0, 1\}^d} \left\{ \sum_{i=1}^n q_i(x) -  \inner{x}{\sigma} \right\}$$
\end{definition}

\end{section}


\begin{section}{Query Release Game}\label{sec:queryreleasegame}
  Given a class of queries $\cQ$ over a database $D$, we want to
  output a differentially private synthetic dataset $\widehat{D}$ such
  that for any query $q\in \cQ$ we have low error:
\begin{align*}
\text{error}(\widehat{D})=\max_{q\in\cQ}|q(D)-q(\widehat{D})| \leq \alpha.
\end{align*}
We revisit a zero-sum game formulation between a data-player and a
query player for this problem \cite{HsuRU13,gaboardi2014dual}. The
data player has action set equal to the data universe $\cX$ and the
query player has action set equal to the query class $\cQ$.  We make
the assumption that $\cQ$ is closed under negation. That is, for every
query $q \in \cQ$ there is a \textit{negated query} $\bar{q} \in \cQ$
where $\bar{q}(D) = 1 - q(D)$. If $\cQ$ is not closed under negation,
we can simply add negated queries to $\cQ$. Since $\cQ$ is closed
under negations, we can write the error as
\[
|q(D) - q(\widehat D)| = \max\{ q(D) - q(\widehat D), \neg q(D) - \neg q(\widehat D) \}
\]
This allows us to define a payoff function that captures the error of
$\widehat D$ without the absolute value. In particular, the payoff for
actions $x\in \cX$ and $q\in\cQ$ is given by:
\begin{equation}\label{eq:payoff}
A(x,q) := q(D) - q(x)
\end{equation}
The data player wants minimizes the payoff $A(x,q)$ while the query
player maximizes it. Intuitively, the data player would like to find a
distribution with low error, while the query player is trying to
identify the query with the worst error. Each player chooses a mixed
strategy, that is a distribution over their action set. Let
$\Delta(\cX)$ and $\Delta(\cQ)$ denote the sets of distributions ove
$\cX$ and $\cQ$.  For any $\widehat D \in \Delta(\cX)$ and
$\widehat Q\in \Delta(\cQ)$, the payoff is defined as
\[
  A(\widehat D, \cdot) = \E{x\sim \widehat D}{A(x, \cdot)}, \quad A(\cdot,
  \widehat Q) = \E{q\sim \widehat Q}{A(\cdot, q)}.
\]
A pair of mixed strategies
$(\widehat D, \widehat Q)\in \Delta(\cX) \times \Delta(\cQ)$ forms an
$\alpha$-approximate equilibrium of the game if
\begin{equation}
  \max_{q\in\cQ} A(\widehat D, q) - \alpha \leq A(\widehat D, \widehat Q) \leq \min_{x\in\cX} A(x, \widehat Q) + \alpha,
\end{equation}

The following result allows us to reduce the problem of query release
to the problem of computing an equilibrium in the game.

\begin{theorem}[\citet{gaboardi2014dual}]
  Let $(\widehat D, \widehat Q)$ be any $\alpha$-approximate
  equilibrium of the query release game, then the data player's
  strategy $\widehat D$ is $2\alpha$-accurate,
  $\text{error}(\widehat{D})=\max_{q\in\cQ}| q(D)-q(\widehat{D})| \leq 2\alpha$.
\end{theorem}

\subsection{No-Regret Dynamics}
To compute such an equilibrium privately, we will simulate no-regret
dynamics between the two players. Over rounds $t = 1 , \ldots , T$,
the two players will generate a sequence of plays
$(D^1, Q^1), \ldots , (D^T, Q^T)\in \Delta(\cX) \times
\Delta(\cQ)$. The regrets of the two players are defined as
\begin{align*}
  R_\data(T) = \frac{1}{T}\left( \sum_{t=1}^T A(D^t, Q^t) -  \min_{x\in \cX} \sum_{t=1}^T A(x, Q^t)\right) \\
  R_\query(T) = \frac{1}{T}\left( \max_{q\in \cQ} \sum_{t=1}^T A(D^t, q) - \sum_{t=1}^T A(D^t, Q^t)\right)
\end{align*}

\begin{theorem}[Follows from \cite{FS97}] \label{meta-theorem}
The average play
  $(\overline D, \overline Q)$ given by
  $\overline D = \frac{1}{T} \sum_{t=1}^T D^t$ and
  $\overline Q = \frac{1}{T} \sum_{t=1}^T Q^t$ from the no-regret
  dynamics above is an $\alpha$-approximate equilibrium with
\[
\alpha =R_\data(T) + R_\query(T).
\]
\end{theorem}

We will now provide two frameworks to obtain regret bounds for the two
players.

\end{section}


\section{Primal Oracle-Efficient
  Framework}\label{sec:noregretdynamics}

In the primal framework, we will have the data player run a online
learning algorithm to update the distributions $D^1, \ldots , D^T$
over rounds and have the query player play an approximate best
response $Q^t$ against $D^t$ in each round. The algorithm \mwem falls
under this framework, but the no-regret algorithm (MW) runs in
exponential time even in the best case since it maintains a
distribution over the entire domain $\cX$. We replace the MW method
with two variants of the follow-the-perturbed-leader (FTPL) algorithm
\cite{KalaiV05}---\algoA \cite{online2019suggala} and \algoB
\cite{syrgkanis2016efficient}. Both of these algorithms can generate a
sample from their FTPL distributions by relying an oracle to solve a
perturbed optimization problem. (In our experiments, we instantiate this oracle with an integer
program solver.) For both algorithms, the query player selects a
query $q_t\in \cQ$ (that is $Q_t$ is point mass distribution on $q_t$) using the exponential mechanism, denoted by $\cM_E$. We present this primal framework in
Algorithm~\ref{alg:framework}.

\begin{algorithm}
\begin{algorithmic}
  \caption{Primal Framework of No-Regret Dynamics}\label{alg:framework}
  \REQUIRE FTPL algorithm $\cA$
\INPUT A dataset $D\in \cX^n$, query class $\cQ$, number of rounds $T$, target privacy $\rho$.
\STATE Initialize $\rho_0 = \rho/T$. Get initial sample $q_0\in\cQ$ uniformly at random.
\FOR{$t = 1 $ {\bfseries to} $T$}
	\STATE \textbf{Data Player} Generate $\widehat{D^t}$ with online learner $\cA$ with queries $q_0, \ldots ,q_{t-1}$.
  \STATE \textbf{Query player:} Define score function $S_t$. For each query $q \in \cQ$, set
	$S_t(D, q) =q(D) - q(\widehat{D^t})$.

	\STATE
  Sample $q_t \sim \cM_E(D, S_t, \cQ, \sqrt{2\rho_0})$ \COMMENT{{such that EM  satisfies $\rho_0$-zCDP}}
	%
\ENDFOR
\OUTPUT $\frac{1}{T} \sum_{t=1}^T \widehat{D^t}$
\end{algorithmic}
\end{algorithm}

Now we instantiate the primal framework above with two no-regret
learners, which yield two algorithms \fem ((Non-Convex)-FTPL with
Exponential Mechanism) and \sfem (Separator-FTPL with Exponential
Mechanism). First, the \fem algorithm at each round $t$ computes a
distribution $D_t$ by solving a perturbed linear optimization problem
polynomially many times. The optimization objective is given by the payoff
against the previous queries and a linear perturbation
\[
\argmin_{x\in \cX} \sum_{i=0}^{t-1} A(x, q_i) + \langle x, \sigma\rangle
\]
where $\sigma$ is a random vector drawn from the exponential
distribution. 
Observe that the first term $q_i(D)$ in $A(x, q_i) = q_i(D) - q_i(x)$ does not depend on $x$. Thus, we can further simplify the objective as
\[
\argmax_{x\in\cX}\left \{ \sum_{i=0}^{t-1} q_i(x) - \langle x, \sigma\rangle \right \}
\]
To solve this problem above, we will use an linear optimization
oracle (\cref{linopt}), which we will implement using an integer program solver.

\begin{algorithm}[h]
\begin{algorithmic}
\caption{Data player update in \fem}
\label{alg:fem}
\INPUT Queries $q_0, \ldots, q_{t-1}\in \cQ$, exponential distribution scale $\eta$, number of samples $s$.
\FOR{$j \gets 1$ \textbf{to} $s$}
\STATE Let $\sigma_{j} \in \mathbb{R}^d$ be a random vector such that each coordinate of $\sigma_{j}$ is drawn from the exponential distribution $\text{Exp}(\eta)$. Obtain a FTPL sample $x_j^t$ by solving
 $$x_j^t \in \argmax_{x\in\cX} \left \{ \sum_{i=0}^{t-1} q_i(x) - \langle x, \sigma_{j}\rangle \right \}$$
\ENDFOR
\OUTPUT $\widehat{D_t}$ as the uniform distribution over $\{ x_1^t,  \ldots , x_s^t \}$\end{algorithmic}
\end{algorithm}

The second algorithm is less general, but as we will show it achieves
a better error rate for important classes of queries.  Algorithm \sfem relies on
the assumption that the query class $\cQ$ has a small separator set
$\sepQ$.

\begin{definition}[Separator Set]
  A queries class $\cQ$ has a small separator set $\sepQ$
   if for any two
  distinct records $x,x'\in \cX$, there exist a query $q:\cX\rightarrow \{0,1\}$ in $\sepQ$ such that
  $q(x) \neq q(x')$.
\end{definition}

Many classes of statistical queries defined over the boolean hypercube
have separator sets of size proportional to their VC-dimension or the
dimension of the input data. For example, boolean conjunctions,
disjunctions, halfspaces defined over the $\{0, 1\}^d$, and parity
functions all have separator sets of size $d$.

Algorithm \sfem then perturbs the data player's optimization problem by inserting
``fake'' queries from the separator set:
\[
  \argmax_{x\in\cX} \left \{ \sum_{i=1}^{t-1} q_i(x) + \sum_{\tilde{q}_j\in \sepQ} \sigma_{j}
    \tilde{q}_j(x) \right \},
\]
where each $\sigma_{j}\in\mathbb{R}$ is sampled from the Laplace distribution. This problem
can be viewed as a simple special case of the linear optimization
problem in \cref{linopt} with no linear perturbation term.

\begin{algorithm}[h]
\begin{algorithmic}
\caption{Data player update in \sfem}
\label{alg:sepfem}
\INPUT Queries $q_0, \ldots, q_{t-1}\in \cQ$, Laplace noise scale $\eta$, number of samples $s$.
\STATE Let $\sepQ = \{\tilde q_1 , \ldots , \tilde q_M\}$ be the serparator set for  $\cQ$.
\FOR{$j =1$ \textbf{to} $s$}
\STATE Let $\sigma_j \in \mathbb{R}^M$ be a fresh random vector such that each coordinate of $\sigma_j$ is drawn from the Laplace distribution $\text{Lap}(\eta)$. Obtain a FTPL sample $x_j^t$ by solving
	\[
        x_j^t \in  \argmax_{x\in\cX} \left \{ \sum_{i=0}^{t-1} q_i(x) + \sum_{i=1}^M
            \sigma_{j, i} \tilde{q}_i(x) \right \}
        \]
\ENDFOR
\OUTPUT $\widehat{D_t}$ be a uniform distribution over $\{ x_1^t, \ldots , x_s^t \}$
\end{algorithmic}
\end{algorithm}

To derive the privacy guarantee of these two algorithms, we observe that
the data player's update does not directly use the private dataset
$D$. Thus, the privacy guarantee directly follows from the composition
of $T$ exponential mechanisms.

%
%
\begin{theorem}[Privacy]\label{thm:frameworkprivacy}
  Algorithm \ref{alg:framework}  satisfies $\rho$-zCDP for any  instantiated with any no-regret algorithm then it
\end{theorem}
\begin{proof}
The \cref{alg:framework} executes $T=\rho/\rho_0$
runs of of the exponential mechanism $\cM(x, S, \cR, \sqrt{2\rho_0})$
with parameter $\sqrt{2\rho_0}$. Then by Lemma \ref{lem:EMprivacy},
we have that $\cM(x, S, \cR,\sqrt{2\rho_0})$
satisfies $\rho_0$-zCDP.
Finally Lemma \ref{lem:zcdpcomposition} states that the composition
of $T = \tfrac{\rho}{\rho_0}$ $\rho_0$-mechanims satisfies $\rho$-zCDP.
\end{proof}



To derive the accuracy guarantee of the two algorithms, we first bound
the regret of the two players. Note that the regret guarantee of the
data player follow from the regret bounds on the two FTPL algorithms
\cite{online2019suggala} and \cite{syrgkanis2016efficient}. The regret
guarantee of the query player directly follows from the utility
guarantee of the exponential mechanism \cite{mcsherry2007mechanism}.
We defer the details to the appendix.

\begin{corollary}[\fem Accuracy]\label{thm:femaccuracy}
Let $d =\log(\cX)$. For any dataset $D\in \cX^n$, query class $\cQ$ and privacy parameter $\rho>0$, there exists $T, \eta$ and $s$ so that with probability at least $1-\beta$, the algorithm \fem finds a synthetic database $\widehat{D}$ that answers all queries in $\cQ$ with error
\begin{align*}
\max_{q\in\cQ}|q(D)-q(\widehat{D})|\leq
\widetilde{O} \left(
\frac{d^{3/4}\sqrt{\log\left(\tfrac{|\cQ|}{\beta}\right)}}{\rho^{1/4}n^{1/2}}
\right)
\end{align*}
By Lemma \ref{lem:cdpTodp}, \cref{alg:fem} satisfies
$(\eps, \delta)$-differential privacy with
$\eps = \rho + 2\sqrt{\rho\log(1/\delta)}$.
If $\eps<1$ then  \fem has error

\begin{align*}
\max_{q\in\cQ}|q(D)-q(\widehat{D})|\leq
  \widetilde{O} \left(
    \frac{d^{3/4} \log^{1/2} |\cQ| \cdot
      \sqrt{
        \log(\tfrac{1}{\delta}) \log(\tfrac{1}{\beta})}
      }{\eps^{1/2}n^{1/2}} \right)
\end{align*}
\end{corollary}
%

\begin{corollary}[\sfem Accuracy]\label{thm:sepfemaccuracy}
Let $d =\log(\cX)$. For any dataset $D\in\cX^n$ and query class $\cQ$ with a separator set $\sepQ$ and privacy parameter $\rho>0$, there exist $T, \eta$ and $s$ so that with probability at least $1-\beta$, algorithm \sfem finds a synthetic  database $\widehat{D}$ that answers all queries in $\cQ$ with error
\begin{align*}
\max_{q\in\cQ}|q(D)-q(\widehat{D})|\leq
 \widetilde{O }\left(
\frac{|\sepQ|^{3/8}d^{1/4} \sqrt{\log\left(\tfrac{|\cQ|}{\beta}\right)}}{\rho^{1/4}n^{1/2}}
\right)
\end{align*}
By Lemma \ref{lem:cdpTodp}, \cref{alg:framework} satisfies
$(\eps, \delta)$-differential privacy with
$\eps = \rho + 2\sqrt{\rho\ln(1/\delta)}$.
If $\eps<1$ then  \sfem has error
\begin{align*}
\max_{q\in\cQ}|q(D)-q(\widehat{D})|\leq
\widetilde{O } \left(
\frac{|\sepQ|^{3/8}d^{1/4}\sqrt{\log\left(\tfrac{|\cQ|}{\beta}\right)
\log\left( \tfrac{1}{\delta}\right)}
 }{\eps^{1/2}n^{1/2}}
\right)
\end{align*}
\end{corollary}

Note that if the query class $\cQ$ has a separator set of size $O(d)$,
which is the case for boolean conjunctions, disjunctions, halfspaces
defined over the $\{0, 1\}^d$, and parity functions, then the bound
above becomes
\begin{align*}
\max_{q\in\cQ}|q(D)-q(\widehat{D})|\leq
\widetilde{O} \left( \frac{d^{5/8} \log^{1/2} |\cQ| \cdot
      \log^{1/2}(1/\delta) \log^{1/2}(1/\beta)} {\eps^{1/2}n^{1/2}} \right)
\end{align*}


\begin{remark}
Non-convex FEM and Separator FEM exhibit a better tradeoff between $\alpha$ and $n$ than DualQuery, but a slightly worse dependence on $d$ compared to DualQuery and MWEM.

\end{remark}

\begin{section}{\dqrs: \dq with Rejection Sampling}

  In this section, we present an algorithm \dqrs that builds on the
  \DQ\ algorithm \cite{gaboardi2014dual} and achieves better provable
  sample complexity. In \DQ, we employ the dual framework of the query
  release game -- the query player maintains a distribution over
  queries using the Multiplicative Weights (MW) no-regret learning
  algorithm and the data player best responds. However, the query
  player cannot directly use the distribution $\cQ^t$ proposed by MW
  during round $t$ because it depends on the private data. Instead,
  for each round $t$, it takes $s$ samples from $\cQ^t$ to form an
  estimate distribution $\widehat{\cQ^t}$. The data player then
  best-responds against $\widehat{\cQ^t}$. Sampling from the MW
  distribution $\cQ^t$ can be interpreted as a sample from the
  exponential mechanism. The sampling step incurs a significant
  privacy cost.

Our algorithm \dqrs~improves the sampling step of \DQ~in order to reduce the privacy cost (and the runtime). The basic idea of our algorithm \dqrs~is to apply the rejection sampling technique to ``recycle'' samples from prior rounds. Namely, we generate some samples from $\cQ^t$ using the samples obtained from the distribution in the previous round, i.e., $\cQ^{t-1}$. This is possible because $\cQ^t$ is close to $\cQ^{t-1}$. We show that by taking fewer samples from $\cQ^t$ for each round $t$, we consume less of the privacy budget. The result is that the algorithm operates for more iterations and obtains lower regret (i.e., better accuracy).

\begin{theorem}
  \label{thm:rsprivacy}
  \dq with rejection sampling (Algorithm \ref{alg:dqrs}) takes in a private dataset $D \in \cX^n$ and makes $T=O\left(\frac{\log|\cQ|}{\alpha^2}\right)$ queries to an optimization oracle and outputs a dataset $\tilde{D}=(x^1, \cdots, x^T)\in \cX^T$ such that, with probability at least $1-\beta$, for all $q \in \cQ$ we have $|q(\tilde{D})-q(D)|\le\alpha$. 
  The algorithm is $(\varepsilon,\delta)$-differentially private and attains accuracy $$\alpha = O\left(\frac{\log(|\cX|T/\beta) \cdot \log^{3}(|\cQ|) \cdot \log(1/\delta)}{n^{2} \varepsilon^{2}}\right)^{1/5}.$$
\end{theorem}

In contrast, \dq (\emph{without} rejection sampling) obtains the same result except with 
$$\alpha = O\left(\frac{\log(|\cX|T/\beta) \cdot \log^{3}(|\cQ|) \cdot \log(1/\delta)}{n^{2} \varepsilon^{2}}\right)^{1/6}.$$ In other words, \dqrs~attains strictly better accuracy than \DQ~for the same setting of other parameters.


\begin{algorithm}[h]
\begin{algorithmic}
\caption{Rejection Sampling Dualquery}\label{alg:dqrs}
\REQUIRE Target accuracy $\alpha \in (0,1)$, target failure probability $\beta \in (0,1)$
\INPUT dataset $D$, and linear queries $q_1, \ldots, q_k \in \cQ$
\STATE Set $T  = \frac{16 \log |\cQ|}{\alpha^2}$, $\eta = \frac{\alpha}{4}$
\STATE $s = \frac{48 \log \left({3 |\cX| T}/{\beta} \right)}{\alpha^2}$
\STATE Construct sample $S_1$ of $s$ queries $\{q_i\}$ from $\cQ$ according to $\cQ^1=\mathsf{Uniform}(\cQ)$
\FOR{$t \gets 1 $ \textbf{ to } $T$}
	\STATE Let $\tilde{q} = \frac{1}{s}\sum_{q \in S_t} q$\;
	\STATE Find $x^t$ with $A_D(x^t, \tilde{q}) \geq \max_{x} A_D(x, \tilde{q}) - \alpha/4$\;
  \STATE Let $\gamma_t = \frac{1}{2t^{2/3}}$
	\FORALL{$q \in \cQ$}
		\STATE $\hat{\cQ}_q^{t+1} :=e^{-\eta-\gamma_t} \cdot  \exp\left(-\eta A_D (x^t, q) \right)\cQ_q^{t}$\;
	\ENDFOR

	\STATE Normalize $\hat{\cQ}^{t+1}$ to obtain $\cQ^{t+1}$
  \STATE Construct $S_{t+1}$ as follows
  \STATE Let $\tilde{s}_t = (2\gamma_t+4\eta) s$ and add $\tilde{s}_t$ independent fresh samples from $Q^{t+1}$ to $S_{t+1}$
	\FORALL{$q \in S_t$}
		\STATE Add $q$ to $S_{t+1}$ with probability $\hat{\cQ_q}^{t+1}/\cQ_q^{t}$
		\STATE If $|S_{t+1}|>s$, discard elements at random so that $|S_{t+1}|=s$
	\ENDFOR
\ENDFOR
\OUTPUT Sample $y_1, \ldots, y_s$
\end{algorithmic}
\end{algorithm}

The analysis of \dqrs~largely follows that of \DQ. The key difference is the analysis of the rejection sampling step, which is summarized by the following two lemmas. The first one shows that taking samples drawn from $Q=\cQ^t$ and performing rejection sampling yields samples from $P=\cQ^{t+1}$; thus $S_{t+1}$ is distributed exactly as if it were drawn from $\cQ^{t+1}$. The second lemma gives a bound on the privacy loss of the rejection sampling step.

\begin{lemma}[Rejection Sampling Accuracy]\label{lem:rs-dist2}
	Let $P$ and $Q$ be probability distributions over $\cQ$, and let $M \ge \max_{q \in \cQ} P_q / Q_q$. Sample an element of $\cQ$ as follows. Sample $q$ according to $Q$, and accept it with probability $P_q / (M\cdot Q_q)$. If $q$ is not accepted, sample $q$ according to $P$. Then the resulting element is distributed according to $P$.
\end{lemma}
\begin{lemma}[Rejection Sampling Privacy] \label{lem:good-samples}
	The subroutine which accepts $q$ with probability $\hat{Q}^{t+1}_q /  Q^t_q = e^{-\eta-\gamma_t} \cdot \exp(-\eta A_D(x^t, q) )$ is $\eps$-differentially private for $\eps = \max\left\{\eta / n, \eta/\gamma_t n\right\}$.
\end{lemma}


%
%

\end{section}

\section{Experiments on the Adult dataset}
\label{sec:experiments}
%
We evaluate the algorithms presented in this paper on two different datasets: the ADULT dataset from the UCI repository \cite{dua:2019} and the LOANS dataset. The datasets used in our experiments are summarized in \cref{table:summary}. For the experiments in this section, we focus on answering $3$-way marginal and $5$-way marginal queries. We ran two sets of experiments. One looks into how well the algorithms scale with the privacy budget, and we test for privacy budget $\eps$ taking value in $0.1, 0.15, 0.2, 0.25, 0.5$, and $1$. The second one looks into how the algorithms' performance degrades when we rapidly increase the number of marginals workload to answer. To measure the accuracy of a synthetic dataset $\widehat{D}$ produced by the algorithm, we used the max additive error over a set of queries $Q$:
%
$\textrm{error}(\widehat{D}) = \max_{q\in Q} |q(D)-q(\widehat{D})|$.
\begin{table}[h]
\caption{Datasets}
\label{table:summary}
\vskip 0.15in
\begin{center}
\begin{small}
\begin{sc}
\begin{tabular}{lcccc}
\toprule
Data set & Records & Attributes & Binary \\
\midrule
ADULT    & 48842 &  15  & 500 \\
LOANS    & 42535 &  48  & 500 \\
\bottomrule
\end{tabular}
\end{sc}
\end{small}
\end{center}
\vskip -0.1in
\end{table}

Our first set of experiments (\cref{fig:vsprivacy}) fix the number of queries and evaluate the performance on different privacy levels. From the first result, we observe that \fem's max error rate increases more slowly than \hdmm's as we increase the privacy level (decrease $\eps$ value). Our second set of experiments (\cref{fig:vsworkload}) fix the privacy parameters and evaluates performance on increasing workload size (or the number of marginals). The results from this section, show that \fem's max error rate increases much more slowly than \hdmm's. From the experiments, we can conclude that at least of the case of $k$-way marginals and dataset ADULT and LOANS, \fem scales better to both the high privacy regime (low $\eps$ value) and the large workload regime (high number of queries) than the state-of-the-art \hdmm method.
\begin{figure}[h]
\centering
\subcaptionbox{ADULT dataset on 3-way marginal queries.\label{adult3eps}}
{\includegraphics[width=0.35\columnwidth]{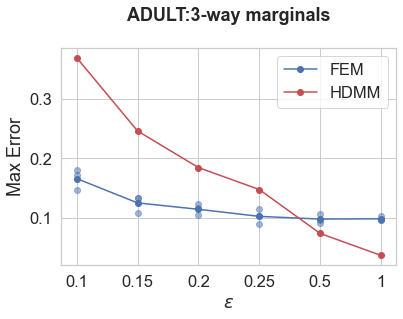}}
\subcaptionbox{LOANS dataset on 3-way marginal queries.\label{loans3eps}}
{\includegraphics[width=0.35\columnwidth]{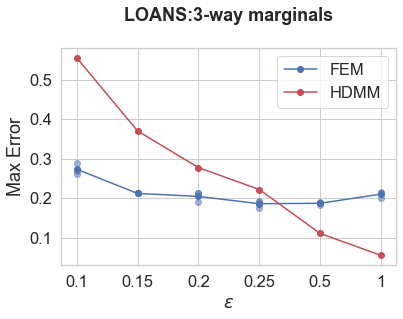}}
\subcaptionbox{ADULT dataset on 5-way marginal queries.\label{adult5eps}}
{\includegraphics[width=0.35\columnwidth]{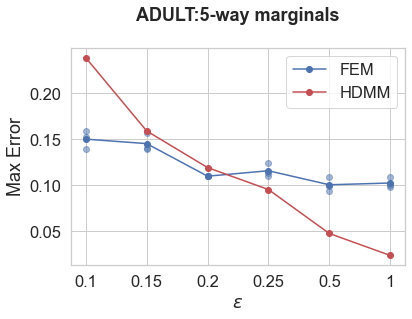}}
\subcaptionbox{LOANS dataset on 5-way marginal queries.\label{loans5eps}}
{\includegraphics[width=0.35\columnwidth]{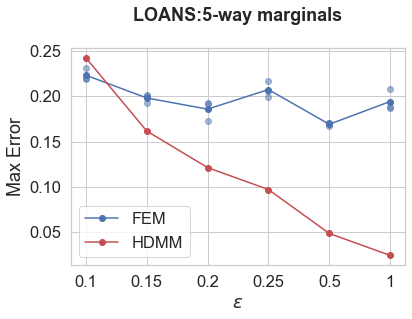}}
\caption{Max-error for 3 and 5-way marginal queries on different privacy levels.
The number of marginals is fixed at 64. We enumerate all queries for each marginal.(see \cref{def:marginals})}\label{fig:vsprivacy}
\end{figure}
\begin{figure}[h]
\centering
\subcaptionbox{ADULT dataset on 3-way marginal queries. \label{adult3}}
{\includegraphics[width=0.35\columnwidth]{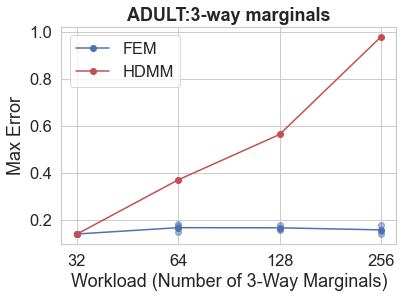}}
\subcaptionbox{LOANS dataset on 3-way marginal queries.\label{loans3}}
{\includegraphics[width=0.35\columnwidth]{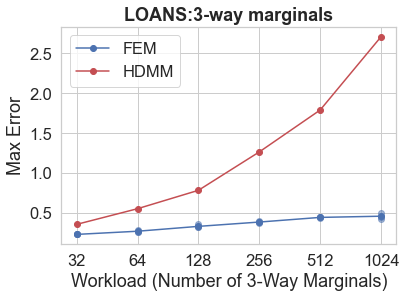}}
\subcaptionbox{ADULT dataset on 5-way marginal queries. \label{adult5}}
{\includegraphics[width=0.35\columnwidth]{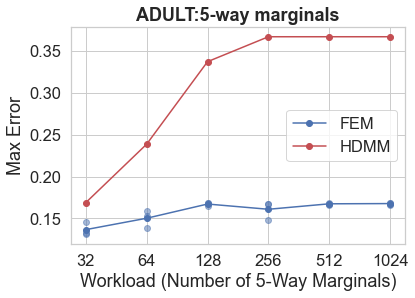}}
\subcaptionbox{LOANS dataset on 5-way marginal queries.\label{loans5}}
{\includegraphics[width=0.35\columnwidth]{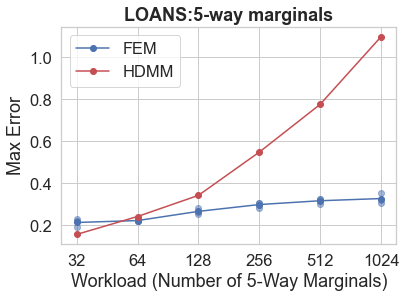}}
\caption{Max-error for increasing number of $3$ and $5$-way marginals. We enumerate all queries for each marginal (see \cref{def:marginals}).
The privacy parameter $\eps$ is fixed at $0.1$ and $\delta$ is $\tfrac{1}{n^2}$, where $n$ is the size of the dataset. .}\label{fig:vsworkload}
\end{figure}

\paragraph{Hyper-Parameter Selection}
In our implementation, algorithm \fem~ has hyperparameters $\eps_0$
and $\eta$. Both the accuracy and the run time of the algorithm depend
on how we choose these hyperparameters. For \fem~, we ran grid-search on different hyperparameter combinations and reported the one with the smallest error.
The \cref{table:femparams} summarizes the range of hyperparameters used for
the first set of experiments in \cref{fig:vsprivacy}.
Then \cref{table:femparamstwo} summarizes the range of hyperparameters used for the second set of experiments in  \cref{fig:vsworkload}.

However, in real-life scenarios, we may not have access to an optimization procedure to select the best set of hyperparameters since every time we run the algorithm, we are consuming our privacy budget.
Therefore, selecting the right combination of hyperparameters can be challenging. We briefly discuss how each parameter affects \fem's performance. The $\eta$ parameter is the scale of the random
objective perturbation term.  The data player samples a
synthetic dataset $\widehat{D}$ from the Follow The Perturbed Leader
distribution with parameter $\eta$ as in algorithm \ref{alg:fem}.  The
perturbation scale $\eta$ controls the rate of convergence of the
algorithm.  Setting this value too low can make the algorithm unstable
and leads to bad performance. If set too high, the solver in FTPL focuses too much on optimizing over the noise term.

The parameter $\eps_0$ corresponds to the privacy consumed on each
round by the exponential mechanism parameterized with $\eps_0$.  The goal is to find a query that maximizes the error on
$\widehat{D}$. Thus, the parameter $\eps_0$ controls the number of iterations. Again we face a trade-off in choosing
$\eps_0$, since setting this value too high can lead to too few
iterations giving the algorithm no chance to converge to a good
solution. If $\eps_0$ is too low, it can make the algorithm run too slow, and also it makes it hard for the query player's exponential mechanism to find queries with large errors.

\begin{table}[h]
\caption{First \fem~ hyperparameters for  \cref{fig:vsprivacy}.}
\label{table:femparams}
\vskip 0.15in
\begin{center}
\begin{small}
\begin{sc}
\begin{tabularx}{\linewidth}{c XX }
\toprule
Param & Description & Range  \\
\midrule
$\eps_0$    & Privacy budget used per round &  0.003, 0.005, 0.007, 0.009, 0.011, 0.015, 0.017, 0.019\\
$\eta$    & Scale of noise for objective perturbation &  1, 2, 3, 4  \\
\bottomrule
\end{tabularx}
\end{sc}
\end{small}
\end{center}
\vskip -0.1in
\end{table}

\begin{table}[h]
\caption{Second \fem~ hyperparameters for \cref{fig:vsworkload}.}
\label{table:femparamstwo}
\vskip 0.15in
\begin{center}
\begin{small}
\begin{sc}
\begin{tabularx}{\linewidth}{c XX }
\toprule
Param & Description & Range  \\
\midrule
$\eps_0$    & Privacy budget used per round &  0.0025, 0.003, 0.0035 \\
$\eta$    & Scale of noise for objective perturbation &  0.75, 1, 1.25  \\
\bottomrule
\end{tabularx}
\end{sc}
\end{small}
\end{center}
\vskip -0.1in
\end{table}

\paragraph{Data discretization} We discretize ADULT and LOANS datasets into binary attributes by mapping each possible value of a discrete attribute to a new binary feature. We bucket continuous attributes, mapping each bucket to a new binary feature.

\paragraph{Optimizing over $k$-way Marginals}
We represent a data record  by its one-hot binary encoding with dimension $d$, thus $\cX = \{0,1\}^d$ is the data domain. On each round $t$ the algorithm \fem~ takes as input a sequence of $t$ queries $\left(q^{(1)}, \ldots, q^{(t)}\right)$ and a random perturbation term $\sigma \sim \Lap(\eta)^d$ and solves the following  optimization problem

\begin{align}\label{eq:femobj}
\argmax_{x\in  \{0,1\}^d} \left \{ \sum_{i=1}^{t-1} q^{(i)}(x) - \langle x, \sigma\rangle \right \}
\end{align}

Let $Q_k$ be the set of $k$-way marginal queries. We can represent any $k$-way marginal query $q\in Q_k$ for $\cX$ in vector form with a $d$-dimensional binary vector $\vec{q}$ such that $\vec{q}\in \{0,1\}^d$ and $\|\vec{q}\|_1 = k$. Then we can define $q\in Q_k$ as
\begin{align*}
  q(x) =
  \begin{cases}
    1 & \text{if } k = \inner{x}{\vec{q}}  \\
    0 & \text{otherwise}
  \end{cases}
\end{align*}
Let $\bar{Q}_k$ be the set of negated $k$-way marginals. Then for any $q\in \bar{Q}_k$
\begin{align*}
  q(x) =
  \begin{cases}
    0 & \text{if } k =  \inner{x}{\vec{q}}  \\
    1 & \text{otherwise}
  \end{cases}
\end{align*}

Next we formulate the optimization problem \cref{eq:femobj} as an integer program. Given a sequence of $t$ queries $\left(q^{(1)}, \ldots, q^{(t)}\right)$ and a random perturbation term $\sigma \sim \Lap(\eta)^d$.
Let $c_i \in \{0,1\}$ be a binary variable encoding whether the query $q^{(i)}$ is satisfied.
\begin{align*}
&\max_{x \in \{0,1 \}^d } \sum_{i=1}^t c_i
- \inner{x}{\sigma}\\
&\mbox{s.t. for all } i \in \{1, \ldots. t\} \\
&\inner{x}{\vec{q}^{(i)}} \geq k c_i  &\text{if } q^{(i)}\in {Q_k}\\
&\inner{\vec{1}_d-x}{\vec{q}^{(i)}}
\geq c_i   &\text{if } q^{(i)}\in \bar{Q}_k
\end{align*}

Finally, we used the Gurobi solver for mixed-integer-programming to implement \fem's optimization oracle.

%





\paragraph{The implementation}  We ran the experiments on a machine with a 4-core Opteron processor and 192 Gb of ram. We made publicly available the see the exact implementations used for these experiments via GitHub. For
\hdmm's~implementation see \url{https://github.com/ryan112358/private-pgm/blob/master/examples/hdmm.py} and for \fem's implementation  see \url{https://github.com/giusevtr/fem}.
%


\section{Conclusion and Future Work}
In this paper, we have studied the pressing problem of efficiently generating private synthetic data. We have presented three new algorithms for this task that sidestep known worst-case hardness results by using heuristic solvers for NP-complete subroutines. All of our algorithms are equipped with formal privacy and utility guarantees and they are oracle-efficient -- i.e., our algorithms are efficient as long as the heuristic solvers are efficient.

There is a very real need for practical private synthetic data generation tools and a dearth of solutions available; the scientific literature offers mostly exponential-time algorithms and negative intractability results. This work explores one avenue for solving this conundrum and we hope that there is further work both extending this line of work and exploring entirely new approaches. Our experimental evaluation demonstrates that our algorithms are promising and supports our theoretical results. However, our experiments are relatively rudimentary. In particular, we invested most time into optimizing the most promising algorithm \fem. An immediate question is whether further optimization of the other two algorithms could yield better results.

\newpage
\bibliography{main}
\bibliographystyle{icml2020}

\appendix
\onecolumn
\newpage

\appendix

\section{Missing Proofs in Section ~\ref{sec:noregretdynamics}}
%
%
%
%

.

This section describes the accuracy analysis of \fem~and \sfem~ in detail. The accuracy proof proceeds in two steps. First we show that the
sample distribution $\widehat{D^t}$ played by the data player is close
the true distribution $D^t$. Then we show that both the query player and and data player are following no-regret strategies.
Then, by Theorem~\ref{meta-theorem}, we
show that algorithms \fem and \sfem find an approximate equilibrium
of the game dynamics described in section \ref{sec:queryreleasegame}.

To bound the deviation error in our sampling from the FTPL
distribution, we use the following Chernoff bound.

\begin{lemma}[Chernoff Bound] \label{def:chernoff} Let
  $X_1, \ldots, X_m$ be i.i.d random variables such that
  $0 \le X_i \le 1$ for all $i$. Let
  $S = \frac{1}{m} \sum_{i=1}^m X_i$ denote their mean and let
  $\mu = \mathbb{E}[S]$ denote their expected mean. Then,
$$\pr{|S - \mu| > t ] \le 2 \exp{(-2m t^2)}}$$
\end{lemma}

\begin{lemma}\label{lemma:datasampling}
Let $\beta \in (0, 1)$ and let $D^t$ be the true distribution over $\cX$. Suppose we draw
$$s  = \frac{8 \log{(4T|\cQ|/\beta)}}{\alpha^2}$$ samples $\{ x_i^t \}$ from $D^t$ to form $\widehat{D^t}$. Then for all $q\in\cQ$,
 with probability at least $1 - \beta / 2$, we have
 \begin{align*}
   \left|\frac{1}{s} \sum_{i=1}^s q(x_i^t)- q(D^t)\right| < \frac{\alpha}{4} \text{ for all } 0 \le t \le T
 \end{align*}
\end{lemma}
\begin{proof}
For any fixed $t$, note that $\frac{1}{s} \sum_{i=1}^s q(x_i^t)$ is the average of the random variables $q(x_1^t), q(x_2^t), \ldots, q(x_s^t)$. Also $\mathbb{E}[q(x^t)] = q(D^t)$ for all $0 \le t \le T$. Thus by the Chernoff bound and our choice of $s$,
$$\pr{\left|\frac{1}{s} \sum_{i=1}^s q(x_i^t)- q(D^t)\right| > \frac{\alpha}{4}} \le 2\exp{(-s \alpha^2 /8)} = \frac{\beta}{2T|\cQ|}$$ A union bound over all $T$ rounds and all
$|\cQ|$ queries gives a total fail probability of at
most $\beta/2$ as desired.
\end{proof}


The query player following the Exponential Mechanism has bounded regret with high probability.
\begin{lemma}[Query Player's Regret]\label{lemma:query-regret}
Let $n$ be the dataset size. For any $\rho>0$, query class $\cQ$, round $T$, and  any sequence of actions $D_1, \ldots, D_T$ by the data
player,  with probability $1-\beta/2$
the query player from algorithm \ref{alg:framework} achieves an
average regret bound of
\begin{align*}
R_\query(T)  \leq
\frac{1}{n}\sqrt{\frac{2T}{\rho}} \log{\left(\tfrac{2T|\mathcal{Q}|}{\beta}\right)}
\end{align*}
\end{lemma}
\begin{proof}
On each round the query player calls the exponential mechanism
with parameter $\sqrt{2\rho_0}$.
Since the sensitivity of the query player's score function $\Delta_S$ is $1/n$, then with probability $1-\beta/2T$ the error for each is round is at most $\frac{2/n}{\sqrt{2\rho_0}} \log{\left(2T|\cQ|/\beta\right)}$ by theorem \ref{thm:exp-error}. Applying union bound over $T$ rounds, with probability $1-\beta/2$ the query player's average regret for $T$ rounds is
\begin{align*}
\max_{q \in \cQ} \frac{1}{T} \sum_{t=1}^T A(\widehat{D^t}, q) - \frac{1}{T} \sum_{t=1}^T A(\widehat{D^t}, q^t)
\leq \frac{1}{T}\sum_{t=1}^T\frac{2/n}{\sqrt{2\rho_0}} \log{\left(2T|\cQ|/\beta\right)}
\le \frac{1}{n}\sqrt{\frac{2T}{\rho}} \log{\left(2T|\mathcal{Q}|/\beta\right)}
\end{align*}
where the last inequality follows from $\rho_0 = \tfrac{\rho}{T}$.
\end{proof}
Now we will provide the accuracy guarantees for \fem and \sfem by
analyzing data player's regret in the two algorithms.

\begin{lemma}[Data Player's Regret in \fem]\label{lemma:data-regret}
Let $d=\log(\cX)$.
For any round $T$ and target accuracy $\alpha>0$, there exist a parameters $\eta$ and $s$ such that if data player from algorithm \fem (\ref{alg:fem}) plays the sequence of distributions
approximations $\widehat{D^1}, \widehat{D^2} \ldots \widehat{D^T}$,
and the query player plays any adversarially chosen sequence of
queries $q_1,\ldots,q_T\in\cQ$,
then the data player, with probability at least $1- \beta/2$, achieves an average regret bound of
\begin{align*}
R_\data^{\mathrm{FEM}}(T)
 \leq \tfrac{\alpha}{4} + \tfrac{5}{2}d^{3/2}\sqrt{\frac{1}{T}}
\end{align*}
\end{lemma}
%


\begin{proof}
For the data player, we use the  Non-Convex-FTPL algorithm for
non-convex losses due to \cite{online2019suggala}.
Recall that, given a sequence of queries $q_1,\ldots,  q_T$ the data player in \cref{alg:framework} wants to choose actions $x_1,\ldots,x_T$ to maximize the objective
\begin{align*}
  \sum_{t=1}^T q_i(x_t)
\end{align*}
thus, the regret of the data player can be writen as
\begin{align*}
R_\data^{\mathrm{FEM}}(T)=
  \frac{1}{T}\max_{x\in\cX}\sum_{t=1}^T q_i(x) -  \frac{1}{T}\sum_{t=1}^T q_i(x_t)
\end{align*}

The results from \cite{online2019suggala} say that if an online learner chooses an action from some decision space with $\ell_\infty$ diameter $D$, the loss functions are $L$-Lipschitz for $\ell_1$ norm, and the learner has access to an $(\alpha, \beta)$-approximate optimization oracle then the learner has expected average regret of the learner bounded by
\begin{align}\label{eq:ncftplregret}
\Ex{{R}(T)} =
 125\eta L d^2 D+\frac{\beta d}{20\eta L}+2\beta d+ \frac{\alpha}{20L}
\end{align}
Suppose that the data player chooses one action on each round by solving the following optimization problem
\begin{align}\label{eq:objective}
  x_t \in \argmin_{x\in\{0, 1\}^d} \left\{ \sum_{i=1}^{t-1} q_i(x) -  \inner{x}{\sigma_t} \right\}
\end{align}
where each $\sigma_t\in \mathbb{R}^d$ is sampled from the exponential distribution, and each $q_i\in\cQ$ is chosen by adversarially.
%
We assume that on each round $t$, the data player plays a single record
$x_t\in \cX$ from the data space $\cX=\{0,1\}^d$  which as $\ell_\infty$
diameter of $1$.
Furthermore, each $q_i$ is $1$-Lipschitz, this follows because each query is bounded in $[0,1]$  and the input are $0$-$1$ vectors from the set $\{0,1\}^d$. Therefore if $x\neq y$ then at last one coordinate in $x$ and $y$ differ by one, hence $\|x-y\|_1\geq 1$. Then the following holds for all $x,y\in \cX$ such that $x\neq y$ and all $q\in\cQ$:
 \begin{align*}
|q(x) - q(y) |\leq 1 \leq \|x-y\|_1
 \end{align*}
We assume that our oracle is a perfect optimizer so $\alpha' = 0$ and
$\beta = 0$. Therefore, we replace constants $D=1$ and $L=1$ in equation
\ref{eq:ncftplregret} and expected regret of the data player is bounded
by
\begin{align*}
\Ex{R(x_1,\ldots,x_T)} =
\mathbb{E}_{x_1,\ldots,x_T} \left[
\frac{1}{T}\max_{x\in\cX}\sum_{t=1}^T q_i(x) -  \frac{1}{T}\sum_{t=1}^T q_i(x_t)
\right]
\leq 125 \eta d^2 + \frac{d}{20\eta T}
\end{align*}
Each $x_t$ is a random variable sampled from its \textit{true distribution} $D^t\in \Delta\cX$, which is given by \cref{eq:objective}.
Now suppose that on each round we could play the true distribution $D^t$ instead of $x_t$, then we can write the regret without the expectation
\begin{align*}
\frac{1}{T}\max_{D\in\Delta\cX}\sum_{t=1}^T q_i(D) -  \frac{1}{T}\sum_{t=1}^T q_i(D^t) \leq 125 \eta d^2 + \frac{d}{20\eta T}
\end{align*}
We want to approximate $D^t$. To that end, the algorithm creates a set $\widehat{D}^t$ of $s$ samples from the distribution $D^t$ by repeatedly calling the optimization oracle with different perturbation values sampled from the exponential distribution with parameter $\eta$.
 From Lemma \ref{lemma:datasampling}, we know that there exist a sample size $s$ such that
 with probability at least $1-\beta/2$, the average error per round of sample $\widehat{D^t}$ from the true distribution $D^t$ is $\alpha/4$. Hence, with probability at least $1-\beta/2$, the average regret per round for the data player playing the sample distribution $\widehat{D^t}$ is
 \begin{align*}
\frac{1}{T}\max_{D\in\Delta\cX}\sum_{t=1}^T q_i(D) -  \frac{1}{T}\sum_{t=1}^T q_i(D^t)
\leq \tfrac{\alpha}{4} + 125 \eta d^2 + \frac{d}{20\eta T}
 \end{align*}
Setting $\eta = \sqrt{\frac{1}{2500Td}}$, we have
\begin{align*}
\frac{1}{T}\max_{D\in\Delta\cX}\sum_{t=1}^T q_i(D) -  \frac{1}{T}\sum_{t=1}^T q_i(D^t)
\leq \tfrac{\alpha}{4} + \sqrt{\left( 125d^2\right)\left(\frac{d}{20T} \right)}
=\tfrac{\alpha}{4} + d^{3/2}\sqrt{\frac{125}{20}\frac{1}{T}}
\end{align*}
\end{proof}

%

\begin{lemma}[Data Player's Regret in \sfem]\label{lemma:sepfemdataregret}
Let $d=\log(\cX)$ and $M=|\sepQ|$.
For any round $T$ and target accuracy $\alpha>0$, there exist a parameters $\eta$ and $s$ such that if data player from algorithm \sfem \ref{alg:sepfem} plays the sequence of distributions
approximations $\widehat{D^1}, \widehat{D^2} \ldots \widehat{D^T}$,
and the query player plays any adversarially chosen sequence of
queries $q_1,\ldots,q_T\in\cQ$,
then the data player, with probability at least $1- \beta/2$, achieves an average expected regret bound of
\begin{align*}
R_\data^{\mathrm{sepFEM}}(T)
 \leq \tfrac{\alpha}{4} +  M^{3/4}d^{1/2}\sqrt{\frac{40}{T}} \end{align*}
\end{lemma}
\begin{proof}
Let $M =| \sepQ | $ be the size of the separator set of the query class and $d = \log(|\cX|)$ is the dimension of the data domain. We use the contextual bandits algorithm on the small separator setting from \cite{syrgkanis2016efficient} which achieves expected regret
\begin{align*}
4\eta M+ \tfrac{10}{\eta}M^{1/2} \log(N) \frac{1}{T}
\end{align*}
where $N$ is the size of the policy space of the learner.

Suppose that the data player chooses $x_t$ on each round $t$, following \cref{alg:sepfem} due to \cite{syrgkanis2016efficient}. In our setting we regard any datum $x_t\in\cX=\{0,1\}^d$ as the policy played by the data player which maps queries to the set $\{0,1\}$. Therefore the policy space has size $2^{d} = |\cX|$. Then according to \cite{syrgkanis2016efficient} and replacing $N$ by $2^d$ we get that the data player achieves expected regret bounded by
\begin{align*}
\mathbb{E}_{x_1,\ldots,x_T} \left[
\frac{1}{T}\max_{x\in\cX}\sum_{t=1}^T q_i(x) -  \frac{1}{T}\sum_{t=1}^T q_i(x_t)
\right]
\leq 4\eta M+ \tfrac{10}{\eta}M^{1/2}  \frac{d}{T}
\end{align*}

Each $x_t$ is a random variable sampled from its \textit{true distribution} $D^t\in \Delta\cX$.
Now suppose that on each round we could play the true distribution $D^t$ instead of $x_t$, then we can write the regret without the expectation
\begin{align*}
\frac{1}{T}\max_{D\in\Delta\cX}\sum_{t=1}^T q_i(D) -  \frac{1}{T}\sum_{t=1}^T q_i(D^t)
\leq
4\eta M+ \tfrac{10}{\eta}M^{1/2} \frac{d}{T}
\end{align*}

We want to approximate $D^t$. To that end, the algorithm creates a set $\widehat{D}^t$ of $s$ samples from the distribution $D^t$ by repeatedly calling the optimization oracle with different perturbation values.

From Lemma \ref{lemma:datasampling}, we know that with probability at least $1-\beta/2$, the average error per round of sample distribution $\widehat{D^t}$ from the true distribution $D^t$ is $\alpha/4$. Hence, with probability at least $1-\beta/2$, the average regret per round for the data player playing the sample distribution $\widehat{D^t}$ is
 \begin{align*}
\frac{1}{T}\max_{D\in\Delta\cX}\sum_{t=1}^T q_i(D) -  \frac{1}{T}\sum_{t=1}^T q_i(D^t)
\leq \tfrac{\alpha}{4} + 4\eta M+ \tfrac{10}{\eta}M^{1/2} d \frac{1}{T}
\end{align*}
Setting $\eta = \sqrt{\frac{5d }{2M^{1/2}T}}$.
Then the regret of the data player is
\begin{align*}
R_\data^{\mathrm{sepFEM}}(T) = \frac{\alpha}{4}+M^{3/4}d^{1/2}\sqrt{\frac{40}{T}}
\end{align*}
\end{proof}

\textbf{Proof of Corollary~\ref{thm:femaccuracy}.}
\begin{proof}
From Lemma \ref{lemma:data-regret} and Lemma \ref{lemma:query-regret}, let $R_\data^{\mathrm{FEM}}(T)$ and $R_\query(T)$ be the upper bounds for the average error of the data and query player respectively with probability at least $1-\beta/2$. Then, with probability at least $1-\beta$ due to the union bound over 2 events, $\alpha$ is the average regret for all rounds by Theorem \ref{meta-theorem}:
\begin{align*}
\alpha &= R_\data^{\mathrm{FEM}}(T) + R_\query(T) \\
& = \frac{\alpha}{4} +\tfrac{5}{2}d^{3/2}\sqrt{\frac{1}{T}}
+ \frac{1}{n}\sqrt{\frac{2T}{\rho}} \log{\left(2T|\mathcal{Q}|/\beta\right)} \\
\end{align*}
To solve for $\alpha$ we first move the first term from the right
hand side. Then we minimize the expression on the left side by
setting the two terms equal to each other.
We ignore the $\log(T)$ term and minimize $\frac{5}{2}d^{3/2}\sqrt{\frac{1}{T}} + \frac{1}{n}\sqrt{\frac{2T}{\rho}}\log(|\cQ|)$ by selecting the correct choice
of $\sqrt{T}$.
That is, setting
$T = \frac{\tfrac{5d^{3/2}}{2}}{\sqrt{\tfrac{2}{\rho n^2}}\log(|\cQ|)}$ we get
\begin{align*}
\tfrac{3\alpha}{4}\leq& \sqrt{\left(\tfrac{5d^{3/2}}{2}\right)
\left( \sqrt{\frac{2}{\rho n^2}\log(|\cQ|)}\right)}\log(2T/\beta) \\
=& \frac{d^{3/4}}{\rho^{1/4}n^{1/2}}\sqrt{\tfrac{5}{2}}\sqrt[4]{2
}\log(2T|\cQ|/\beta)
\end{align*}
%
%
\end{proof}

\textbf{Proof of Corollary~\ref{thm:sepfemaccuracy}.}
\begin{proof}

From lemma \ref{lemma:sepfemdataregret} we have that the
data player's average regret for round $T$  is
$R_\data^{\mathrm{sepFEM}}(T) = M^{3/4}d^{1/2}\sqrt{40} T^{-1/2}$ and the average regret for the query player is $R_\query(T)$ given by lemma \ref{lemma:query-regret}.
Then, by union bound and by Theorem \ref{meta-theorem}, with probability at least $1-\beta$,  the accuracy of \sfem is:
\begin{equation*}
\begin{split}
\alpha &= R_\data^{\mathrm{sepFEM}}(T) + R_\query(T) \\
&=\frac{\alpha}{4}+ M^{3/4}d^{1/2}\sqrt{40} T^{-1/2}+
\frac{1}{n}\sqrt{\frac{2T}{\rho}} \log{\left(2T|\mathcal{Q}|/\beta\right)} \\
\end{split}
\end{equation*}

Now to choose $T$ optimally we ignore the $\log$ term and set
$T=\frac{M^{3/4}d^{1/2}\sqrt{40}}{\sqrt{2/\rho n^2}\log(|\cQ|)}$ to get
\begin{align*}
\frac{3}{4}\alpha \leq   2\sqrt[4]{5}\frac{M^{3/8}d^{1/4} \sqrt{\log(|\cQ|)}}{n^{1/2}\rho^{1/4}}
\end{align*}

\end{proof}

\section{\dqrs: \dq with Rejection Sampling}

\begin{theorem}
  \label{thm:rsprivacy}
  \dq with rejection sampling (Algorithm \ref{alg:dqrs}) takes in a private dataset $D \in \cX^n$ and makes $T=O\left(\frac{\log|\cQ|}{\alpha^2}\right)$ queries to an optimization oracle and outputs a dataset $\tilde{D}=(x^1, \cdots, x^T)\in \cX^T$ such that, with probability at least $1-\beta$, for all $q \in \cQ$ we have $|q(\tilde{D})-q(D)|\le\alpha$. The algorithm is $\rho$-CDP for $$\rho=O\left( \frac{\log(|\cX|T/\beta) \cdot \log^3(|\cQ|)}{n^2 \alpha^5}\right).$$
\end{theorem}

In contrast, \dq (\emph{without} rejection sampling) obtains the same result except with $$\rho = O\left(\frac{\log(|\cX|T/\beta) \cdot \log^3(|\cQ|)}{n^2 \alpha^7}\right).$$

To obtain $(\varepsilon,\delta)$-differential privacy, it suffices to have $\rho$-CDP for $\rho=\Theta(\varepsilon^2/\log(1/\delta)$. Thus the guarantee of Theorem \ref{thm:rsprivacy} can be rephrased as the sample complexity bound $$n=O\left(\frac{\log^{1.5}(|\cQ|) \cdot \sqrt{ \log(|\cX|T/\beta) \cdot \log(1/\delta)}}{\alpha^{2.5} \varepsilon}\right)$$ to obtain $\alpha$-accurate synthetic data with probability $1-\beta$ under $(\varepsilon,\delta)$-differential privacy.

\begin{lemma} \label{lem:good-samples}
The subroutine which accepts $q$ with probability $\hat{Q}^{t+1}_q /  Q^t_q = e^{-\eta-\gamma_t} \cdot \exp(-\eta A_D(x^t, q) )$ is $\eps$-differentially private for $\eps = \max\left\{\eta / n, \eta/\gamma_t n\right\}$.
\end{lemma}
\begin{proof}
Note that $0 < p:=\hat{Q}^{t+1}_q /  Q^t_q = e^{-\eta-\gamma_t} \cdot \exp(-\eta A_D(x^t, q) ) \le e^{-\gamma_t} < 1$. In particular, the probability is well-defined.

We compute the ratio between the probabilities that $q$ is accepted under executions of the algorithm on neighboring datasets $D, D'$ for fixed choices of the best responses $x^1, \dots, x^t$. This ratio is given by
\[
\frac{p}{p'}=\frac{\hat{Q}^{t+1}_q[D]}{Q^t_q[D]} \cdot \frac{Q^{t}_q[D']}{\hat{Q}^{t+1}_q[D']}
= \frac{\exp(-\eta A_D(x^t, q))}{\exp(-\eta A_{D'}(x^t, q))}
\le e^{\eta / n}.
\]
Similarly, we evaluate the ratio of the probabilities that $q$ is \emph{not} accepted under executions of the algorithm on $D$ and $D'$:
Since $p' \le e^{-\gamma_t}$ and $p/p' \ge e^{-\eta/n}$, we have
\[
\frac{1-p}{1-p'} = 1 + \frac{1}{1/p'-1}\left(1-\frac{p}{p'}\right) \le 1 + \frac{1- e^{-\eta/n}}{e^{\gamma_t}-1} \le 1 + \frac{\eta/n}{\gamma_t} \le e^{\eta/\gamma_t n},
\]
as required. 
\end{proof}

Bad samples also incur privacy loss from sampling from the distribution $Q^t$. Just as in~\cite{gaboardi2014dual}, we use the fact that this step can be viewed as an instantiation of the exponential mechanism with score function $\sum_{i = 1}^{t-1}(q(D) - q(x^i))$ to obtain:
\begin{lemma} \label{lem:bad-samples}
Sampling from $Q^{t}$ is $\eps$-differentially private for $\eps = 2\eta(t-1)/n$.
\end{lemma}

\textbf{Proof of Privacy for Theorem~\ref{thm:rsprivacy}.}
\begin{proof}

Each round $t$ incurs privacy loss from $s$ invocations of a $(\eta/\gamma_t n)$-differentially private algorithm (rejection sampling, Lemma~\ref{lem:good-samples}), and $\tilde{s}_t$ invocations of a $(2\eta(t-1)/n)$-differentially private algorithm (Lemma~\ref{lem:bad-samples}).
Since $\varepsilon$-differential privacy implies $\frac12\varepsilon^2$-CDP \cite{BunS16}, we have (by composition) that round $t$ is $\rho_t$-CDP for $$\rho_t = \frac{\eta^2}{2\gamma_t^2n^2}s + \frac{2\eta^2(t-1)^2}{n^2}\tilde{s}_t = \frac{\eta^2s}{n^2}\left(\frac{1}{2\gamma_t^2} + 2(t-1)^2 \cdot (2\gamma_t+4\eta)\right) \le \frac{\eta^2s}{n^2}\left(4t^{4/3} + 8\eta t^2\right).$$
Composing over rounds $t=1 \cdots T$ yields $\rho=O\left( \frac{\log(|\cX|T/\beta) \cdot \log^{2+1/3}(|\cQ|)}{n^2 \alpha^{4+2/3}} + \frac{\log(|\cX|T/\beta) \cdot \log^3(|\cQ|)}{n^2 \alpha^5}\right)$, as required.

\end{proof}

\subsection*{Accuracy}

The accuracy analysis follows that of of \dq, together with the following claims showing that the rejection sampling process simulates the collection of independent samples in the \dq algorithm.

\begin{lemma}\label{lem:rs-dist}
Let $P$ and $Q$ be probability distributions over $\cQ$, and let $M \ge \max_{q \in \cQ} P_q / Q_q$. Sample an element of $\cQ$ as follows. Sample $q$ according to $Q$, and accept it with probability $P_q / (M\cdot Q_q)$. If $q$ is not accepted, sample $q$ according to $P$. Then the resulting element is distributed according to $P$.
\end{lemma}
\begin{proof}
The total probability of sampling $q$ according to this procedure is given by
\begin{align*}
Q_q \cdot \frac{P_q}{M \cdot Q_q} + P_q \cdot \sum_{q' \in \cQ} Q_{q'} \cdot \left(1 - \frac{P_{q'}}{M \cdot Q_{q'}} \right) &= P_q \cdot \left(\frac{1}{M} + \sum_{q' \in \cQ} \left(Q_{q'} - \frac{P_{q'}}{M} \right) \right) \\
&= P_q \cdot \left(\frac{1}{M} +  \left(1 - \frac{1}{M} \right) \right) \\
&= P_q.
\end{align*}
\end{proof}

\begin{lemma}\label{lem:rs-bound}
For any given round $t$, the probability that more than $\tilde{s}_t$ samples are rejected is at most $(e/4)^{\tilde{s}_t} \le \frac{\beta}{3T}$.
\end{lemma}
\begin{proof}
The probability that any given sample is rejected is $1-\hat{Q}^{t+1}_q /  Q^t_q = 1-e^{-\eta-\gamma_t} \cdot \exp(-\eta A_D(x^t, q) ) \le 1-e^{-2\eta-\gamma_t} \le 2\eta+\gamma_t = \frac{\tilde{s}_t}{2s}$. (In particular, $\tilde{s}_t$ is at least twice the expected number of rejected samples.) The set of $s$ samples is rejected independently.
By a multiplicative Chernoff bound, the probability that more than $\tilde{s}_t$ samples are rejected is at most $(e/4)^{\tilde{s}_t}$. Note that $\tilde{s}_t \ge 4\eta s = \frac{48}{\alpha} \log\left(\frac{3|\cX|T}{\beta}\right)$. Thus $(e/4)^{\tilde{s}_t} \le \left(\frac{\beta}{3|\cX|T}\right)^{18/\alpha} \le \frac{\beta}{3T}$.
\end{proof}

Together Lemmas \ref{lem:rs-dist} and \ref{lem:rs-bound} show that, with high probability, at each round $t$, the set $S_t$ is distributed as $s$ independent samples from $\cQ^t$. Given this, the rest of the proof follows that of the original \dq.

\textbf{Proof of Accuracy for Theorem~\ref{thm:rsprivacy}.}
\begin{proof}
For each round $t$, by Hoeffding's bound and Lemma \ref{lem:rs-bound} and a union bound over $\cX$, with probability at least $1-\frac{\beta}{T}$, we have $$\forall x \in \cX \qquad \left|\frac{1}{s} \sum_{q \in S_t} q(x) - \underset{q \leftarrow \cQ^t}{\mathbb{E}}\left[ q(x) \right]\right| \le \frac{\alpha}{4}.$$
By a union bound over the $T$ rounds we have that the above holds for all $t \in [T]$ with probability at least $1-\beta$.

By assumption, in each round $t$, our oracle returns $x^t$ that is an $\alpha/4$-approximate best response to the uniform distribution over $S_t$. Thus, with high probability, the sequence $x^1, \cdots, x^T$ are $\alpha/2$-approximate best responses to the distributions $Q^1, \cdots, Q^t$. Since the distributions are generated by multiplicative weights, we have that this is an $\alpha$-approximate equilibrium. Hence the uniform distribution over $x^1, \cdots, x^T$ is an $\alpha$-accurate synthetic database for $D$.
\end{proof}




\end{document}